\documentclass{ifacconf}

\usepackage[english,american]{babel}

\usepackage{url}
\usepackage{epsfig} 
\usepackage{times} 
\usepackage{amsmath}
\usepackage{amsfonts}
\usepackage{amssymb}  
\usepackage{algorithm}
\usepackage{algpseudocode}

\usepackage{xcolor}
\usepackage{transparent}
\usepackage{graphicx}
\graphicspath{{media/}}
\usepackage{tikz}
\usetikzlibrary{shapes,arrows}
\usetikzlibrary{positioning}

\usepackage{epstopdf}
\usepackage{units}
\usepackage{pgfplots}

\usepackage{pstool}

\graphicspath{{media/}}
\newcommand{\executeiffilenewer}[3]{%
 \ifnum\pdfstrcmp{\pdffilemoddate{#1}}%
 {\pdffilemoddate{#2}}>0%
 {\immediate\write18{#3}}\fi%
}

\newenvironment{proof}[1][Proof]{\begin{trivlist}
\item[\hskip \labelsep {\bfseries #1}]}{\end{trivlist}}

\newtheorem{proposition}{Proposition}[section]
\newtheorem{lemma}{Lemma}[section]

\pgfplotsset{every tick label/.append style={font=\footnotesize}}
\pgfdeclarelayer{background}
\pgfsetlayers{background,main}

\usepackage{graphicx}      
\usepackage{natbib}        
\begin{document}

\newcommand{\prox}{{\mathop{\mathrm{prox}}}}
\newcommand{\argmin}{\mathop{\mathrm{argmin}}}
\newcommand{\diag}{\mathop{\mathrm{diag}}}

\newcommand{\TT}{\ensuremath{\mathsf{\tiny{T}}}}
\newcommand{\T}{^{\TT}}
\newcommand{\Rot}[2][]{{#1}_{\textsc{\scriptsize{#2}}}}
\newcommand{\dRot}[2][]{\dot{#1}_{\textsc{\scriptsize{#2}}}}

\newcommand{\V}[2][]{{^{\textsc{\scriptsize{#2}}}{#1}}}
\newcommand{\dV}[2][]{{^{\textsc{\scriptsize{#2}}}\dot{#1}}}
\newcommand{\uV}[2][]{{_{\textsc{\scriptsize{#2}}}\vec{e}_{#1}}}

\newcommand{\tmat}[1]{\widetilde{#1}}
\newcommand{\tmatt}[2][]{{^{\textsc{\scriptsize{#2}}}\widetilde{#1}}}
\newcommand{\dtmatt}[2][]{{^{\textsc{\scriptsize{#2}}}\dot{\widetilde{#1}}}}

\newcommand{\sym}[1]{\mathrm{symm}(#1)}
\newcommand{\skw}[1]{\mathrm{skew}(#1)}

\newcommand{\diff}[1][]{\mathrm{d}#1}
\newcommand{\dt}{\diff t }

\newcommand{\Vct}[1]{\mathrm{vec}(#1)}

\newcommand{\lmax}{\lambda_{\mathrm{max}}}
\newcommand{\lmin}{\lambda_{\mathrm{min}}}

\newcommand{\tr}[1]{\mathrm{tr}(#1)}

\newcommand{\E}{{\mathop{\mathrm{E}}}}

\begin{frontmatter}

\title{Adaptive Decision-Making with Constraints and Dependent Losses: Performance Guarantees and Applications to Online and Nonlinear Identification} 

\thanks[footnoteinfo]{The author thanks Fang Nan for fruitful discussions that led to parts of this work. He also thanks the German Research Foundation and the Branco Weiss Fellowship, administered by ETH Zurich, for the support.}

\author[First]{Michael Muehlebach} 

\address[First]{Max Planck Institute for Intelligent Systems, Max Planck Ring 4, 72076 Tuebingen, Germany (e-mail: michaelm@tue.mpg.de).}

\begin{abstract}                
: We consider adaptive decision-making problems where an agent optimizes a cumulative performance objective by repeatedly choosing among a finite set of options. Compared to the classical prediction-with-expert-advice set-up, we consider situations where losses are constrained and derive algorithms that exploit the additional structure in optimal and computationally efficient ways. Our algorithm and our analysis is instance dependent, that is, suboptimal choices of the environment are exploited and reflected in our regret bounds. The constraints handle general dependencies between losses (even across time), and are flexible enough to also account for a loss budget, which the environment is not allowed to exceed. The performance of the resulting algorithms is highlighted in two numerical examples, which include a nonlinear and online system identification task. 
\end{abstract}


\end{frontmatter}

\section{Introduction}
The multiplicative weights (or hedge) algorithm is a cornerstone of online learning. It not only provides a unified framework for many important algorithms in computer science (see \cite{Arora}), but represents the basis for online convex optimization, \citep{HazanOnlineConvex} and bandit algorithms, \citep{BanditAlgorithms}. However, despite being both simple and successful, it has arguably received little attention in the control community. Indeed, the fact that the algorithm often deals with decisions over a fixed and finite set of choices and involves either a stationary or a strongly adversarial environment can be perceived as show-stoppers for control problems, which involve dynamics, transients, and decisions over continuous spaces. The article shows that, despite some of these seemingly fundamental differences, the adaptive-decision making framework has the potential to bring new perspectives to control theory and offers flexible tools for designing computationally efficient algorithms with theoretical guarantees. This will be highlighted by introducing a new variant of multiplicative weights that accounts for constraints. These constraints can be used to model a-priori knowledge, uncertainty, dependencies across time, and can account for a loss budget of the environment. We also show that the concept of regret minimization provides a natural framework for nonlinear and online system identification problems.

\emph{Related work:} Adaptive decision-making and online learning has a very rich history, \citep{Bianchi}, \citep{Littlestone}, \citep{Shalev-Shwartz}, due to its connections to game theory, economics, machine learning, and optimization. In its basic set-up, an agent chooses among a finite set of options at every iteration. The environment, which chooses the losses, is assumed to be adversarial, resulting in the typical non-asymptotic regret bound $\sqrt{T \text{log}(m)/2}$ \citep{GTalive}. More recently, various extensions of this setting have been considered. For example, when both, the agent and the environment, make their choices according to the same (randomized) procedure, see e.g. \cite{Schapire}, faster convergence is achieved, resulting in a regret that scales with $\log(m) \log(T)^4$, \citep{Daskalakis2}. A different approach has been taken in \cite{Rakhlin}, where either deterministic or stochastic constraints are imposed on the adversary. The work characterizes the impact of constraints on online learning rates that are theoretically achievable. An even more general, non-parametric setting has been considered in \cite{Daskalakis}, where the authors quantify the impact of constraints and construct an explicit algorithm based on a combination of (optimistic) multiplicative weights and a generalization of the Standard Optimal Algorithm, \citep{Littlestone2}.

The aim of our approach is different compared to these important works. Instead of exposing fundamental relations between the complexity of constraints and the attainable regret in the worst case, we provide an instance-dependent algorithm that is able to exploit knowledge about the constraints, while benefiting from a suboptimal play of the environment. We show that our algorithm is optimal in a precise min-max sense when the time horizon is large. Moreover, the algorithm recovers multiplicative weights in the absence of constraints and for an adversarial environment. Our research is in line with recent work in reinforcement learning \citep{Foster,AdaSearch}; and we will see that both aspects, the ability to cope with constraints and the instance-dependency will lead to vast improvements over multiplicative weights.

\emph{Structure:} The article is structured as follows. Sec.~\ref{Sec:Prelim} formulates the adaptive decision-making problem and discusses the multiplicative weights algorithm. Sec.~\ref{Sec:DMC} presents our variant of multiplicative weights that can handle constraints. Sec.~\ref{Sec:Intervals} analyzes box constraints, which provides important intuition and optimality guarantees for our algorithms. The article concludes with two numerical examples in Sec.~\ref{Sec:Numerics} and a summary in Sec.~\ref{Sec:Conclusion}.

\section{Preliminaries}\label{Sec:Prelim}
We start by summarizing the classical adaptive decision-making framework and the multiplicative weights algorithm. This sets the stage for the later developments.

\subsection{Problem formulation}
We consider the set-up described in Alg.~\ref{Alg:setup}, where an agent has to choose among a set of $m$ options during $T$ iterations (both $m$ and $T$ are positive integers). At iteration $t$, the agent has access to the history of past losses, which are represented by the vectors $l_j \in \mathbb{R}^m$, $j=0,1,\dots,t-1$, as well as the set $X_t$ that constrains $l_t$. He chooses the option $a_t$, the environment then reveals $l_t\in X_t$, and inflicts the loss $l_t^{a_t}$ on the agent. In general, the sets $X_t$ change dynamically over time and may be influenced by past losses and actions, with the only requirement that at time $t$, the agent has knowledge of $X_{t}$. Throughout the article, we use subscripts to denote time (or the iteration number) and superscripts to denote the elements of a vector. As we will see, the agent has a strong incentive to randomize his choices. We therefore introduce the variable $p_t\in \Delta_m$ to represent the agent's probability distribution over the options $\{1,\dots,m\}$ at time $t$, where $\Delta_m$ refers to the $m$-dimensional probability simplex. More precisely, we have
\begin{equation*}
\Pr(a_t=i~|~l_{t-1},\dots,l_0)=p_t^i,
\end{equation*}
and as a result, the expected loss $\E[l_t^{a_t}~|~l_{t-1},\dots,l_0]$ of the agent at time $t$ can be expressed concisely as $p_t\T l_t$.
The agent seeks to minimize regret, which is defined as
\begin{equation}
\sum_{t=0}^{T-1} p_t\T l_t - \min_{i\in \{1,\dots,m\}} \sum_{t=0}^{T-1} l_t^i. \label{eq:regret}
\end{equation}
The regret compares the expected performance of the agent (conditioned on the past) to the performance of the best single option had the agent known all the losses $l_t, t=0,\dots, T-1$ beforehand. At first sight, this performance metric seems weak, since it does not compare against strategies that switch between options. However, we will see that in system identification, for example, there is really one single optimal option, which amounts to the best model among a set of candidates. Moreover, our formalism can also account for a small number of strategies that switch (for example, play option 1 if $t$ is even and option 2 if $t$ is odd), since these can simply be added to the set of options $\{1,\dots,m\}$. In the following, we will derive uniform upper bounds on \eqref{eq:regret}; uniform in the sense that the bounds hold for any strategy of the environment (any sequence $l_0,l_1,\dots$). This includes randomized strategies, or even strategies that depend on the agent's past choices $a_{t-1},\dots,a_0$. We will also derive lower bounds that conversely take any potential strategy of the agent into account. Compared to the traditional setting \citep{Bianchi} the losses of the environment are restricted to the sets $X_t$, which are both known to the environment and the agent.

\begin{algorithm}
\caption{Adaptive decision-making framework with constraints.}
\begin{algorithmic}
\For{$t=0~\mathrm{to}~T-1$}
	\State 1. agent chooses $a_t\in \{1,2,\dots,m\}$ 
	\State \qquad $\rhd~ a_t$ may be the outcome of a random event
	\State \qquad $\rhd~ a_t$ may depend on $l_{t-1}$, $l_{t-2},\dots, l_0$
	\State \qquad $\rhd~ a_t$ may depend on $X_t, X_{t-1}, \dots, X_0$
	\State
	\State 2. environment reveals loss $l_t\in X_t \subset \mathbb{R}^m$
	\State \qquad $\rhd~ l_t$ may be the outcome of a random event
	\State \qquad $\rhd~ l_t$ may depend on $a_{t-1},\dots,a_0$
	\State 
	\State 3. decider incurs loss $l_t^{a_t}$
\EndFor
\State
\State \textbf{Objective agent:}$\sum\limits_{t=0}^{T-1} p_t\T l_t - \!\!\min\limits_{i\in \{1,\dots,m\}} \sum\limits_{t=0}^{T-1} l_t^i\rightarrow \min$
\end{algorithmic}
\label{Alg:setup}
\end{algorithm}

\subsection{The multiplicative weights (hedge) algorithm}
The following subsection summarizes the multiplicative weights algorithm, which is sometimes also referred to as hedge. The multiplicative weights algorithm is designed for $X_0=X_1=\dots=X_{T-1}=[0,L]\times \dots \times [0,L]$, where $L>0$ is a (potentially large) constant. This represents an adversarial setup, where the agent has essentially no prior knowledge except that losses are positive. (It turns out that minor extensions can handle situations where $L$ is unknown a-priori and could become arbitrarily large.) The multiplicative weights algorithm is summarized in Alg.~\ref{Alg:MW} and is particularly appealing due to its simplicity. At each iteration, the weights $w_t^i$ are formed and represent the performance of option $i$ up to iteration $t-1$. The probability distribution $p_t^i$ is chosen to be proportional to $w_t^i$, and can be motivated as follows:
\begin{equation}
p_t:= \argmin_{q\in \Delta_m} q\T \sum_{j=0}^{t-1} l_j - \frac{1}{\epsilon} H(q), \label{eq:update}
\end{equation}
where $H(q):=\sum_{i=1}^{m} -q_i \log(q_i)$ denotes the Shannon entropy of the distribution $q\in \Delta_m$. The objective function in \eqref{eq:update} balances the observed performance over the past iterations with the entropy of $q$, which accounts for the uncertainty over the environment's choice of $l_t$. The amount of entropy regularization is controlled with the parameter $\epsilon$ and varies with the number of rounds that are played (if $T$ is large, the weight on the regularization is increased, if $T$ is small the weight is decreased). 

\begin{algorithm}
\caption{Multiplicative weights}
\begin{algorithmic}
\State $\epsilon \gets \sqrt{8 \log(m)/T}/L$
\For{$t=0~\mathrm{to}~T-1$}
	\State $w_{t-1}^i \gets \exp(-\epsilon \sum_{j=0}^{t-1} l_j^i)$ \Comment{$w_{-1}^i=1$ for $t=0$}
	\State $p_t^i \gets w_{t-1}^i/(\sum_{i=0}^{m} w_{t-1}^i)$
	\State Sample $a_t$ according to $p_t$
\EndFor
\end{algorithmic}
\label{Alg:MW}
\end{algorithm}

\subsection{Analysis of the multiplicative weights algorithm}\label{Sec:MW}
This subsection analyzes the regret that follows from Alg.~\ref{Alg:MW} in a setting where $X_0=\dots=X_{T-1}=[0,L] \times \dots \times [0,L]$. This represents the starting point for the subsequent derivations.

The analysis of Alg.~\ref{Alg:MW} hinges on the (Lyapunov) function $\phi_t:=\sum_{i=1}^{m} w_t^i$. The reason why the analysis of $\phi_t$ is particularly useful is that $\phi_t$ is lower bounded by any $w_t^j, j=1,2,\dots,m$, and, in particular, by the weight corresponding to the best single option in hindsight. Moreover, $\phi_t$ relates to $\phi_{t-1}$ in the following way:
\begin{align}
\phi_t&=\sum_{i=1}^{m} w_t^i=\phi_{t-1} \sum_{i=1}^{m} \frac{w_{t-1}^i}{\phi_{t-1}} \exp(-\epsilon l_t^i) \label{eq:rec1}\\
&=\underbrace{\phi_{t-1} \exp(-\epsilon l_t\T p_t)}_{\text{part~i}} \underbrace{\sum_{i=1}^{m} \frac{w_{t-1}^i}{\phi_{t-1}} \exp(-\epsilon \sum_{j=1}^{m} p_t^j (l_t^i-l_t^j))}_{\text{part~ii}} \nonumber,
\end{align}
where we have used the definition of $w_t^i$ in the second step, before factoring out $\exp(-\epsilon l_t\T p_t)$ in the last step. Up to this point we have not used the fact that $w_{t-1}^i$ and $p_t^i$ are related, as for example prescribed by Alg.~\ref{Alg:MW}. Thus, the above equation for $\phi_t$ holds for any choice of losses and any $p_t\in \Delta_m$. We also note that part i of the expression captures the running cost of the agent, i.e., the first term in \eqref{eq:regret}, whereas $\phi_t$ dominates the performance of the best single option up to time $t$ and relates to the second term in \eqref{eq:regret}. Hence, part ii, which is the ratio between $\phi_t$ and part i, directly determines the regret, and we will therefore design $p_t$, such that part ii is small. To this end, we reformulate part ii by applying Taylor's theorem to the exponential
\begin{align}
\text{part~ii}&=1 - \epsilon \sum_{i=1}^{m} \sum_{j=1}^m \frac{w_{t-1}^i}{\phi_{t-1}} p_t^j (l_t^i-l_t^j) \nonumber\\
&+ \frac{\epsilon^2}{2} \sum_{i=1}^{m} \frac{w_{t-1}^i}{\phi_{t-1}} \left( \sum_{j=1}^{m} p_t^j (l_t^i-l_t^j) \right)^2 + \frac{\epsilon^3}{6} \rho_t, \label{eq:importantEq}
\end{align}
where $\rho_t$ captures terms of order $\epsilon^3$ and higher. We note that $\rho_t$ is bounded by $L^3 \exp(\epsilon L)$. In order to simplify the presentation we assume $\epsilon L \leq 1$, which is obtained for $T\geq 8\text{log}(m)$. (The assumption can be avoided by applying Hoeffding's lemma to \eqref{eq:rec1}, which, as we will see in the next section, is too coarse for our purposes.) As discussed above, the key for obtaining a small regret is to design $p_t$ such that part ii is small. In multiplicative weights, this is achieved with $p_t^i=w_{t-1}^i/\phi_{t-1}$, which ensures that the \emph{first-order} term vanishes due to the fact that the summand $p_{t}^i p_t^j (l_t^i-l_t^j)$ is skew symmetric. Hence, the environment's choice of $l_t$ can at most affect second-order terms. We further note that the second-order term describes nothing but the variance of $l_t^{a_t}$, which by applying Popoviciu's inequality, can be bounded by $L^2/4$. (The bound is attained if the adversary chooses the losses $l_t^i$ in such a way that half of the probability $p_t$ is concentrated on losses $l_t^i=0$ and half of the probability $p_t$ is concentrated on losses $l_t^i=L$.) This yields therefore
\begin{equation*}
\text{part~ii} \leq 1 + \epsilon^2 L^2/8 + \epsilon^3 \rho_t/6.
\end{equation*}
As we will see in the following, terms of order $\epsilon^3$ are essentially irrelevant, as they contribute at most to additive constants in the final regret bound.
We continue by bounding the right-hand side, where we exploit the assumption $\epsilon L \leq 1$ and $\rho_t\leq L^3 \exp(1)$ to obtain
\begin{equation*}
\text{part~ii} \leq 1+ \epsilon^2 L^2/8 + \epsilon^3 L^3/2 \leq \exp(\epsilon^2 L^2/8+ \epsilon^3 L^3/2),
\end{equation*}
where we have again used a Taylor expansion of the exponential in the last step. Unrolling the recursion for $\phi_t$, see \eqref{eq:rec1}, yields therefore $\max_{i\in \{1,\dots,m\}} w_{T-1}^i \leq \phi_{T-1}$ and
\begin{multline*}
\phi_{T-1} \leq \phi_{-1} \exp(-\epsilon \sum_{t=0}^{T-1} l_t\T p_t + \epsilon^2 L^2 T/8 + \epsilon^3 L^3 T/2).
\end{multline*}
We now rearrange terms, note that $\phi_{-1}=m$, and insert the definition of $\epsilon$, which yields finally
\begin{align*}
\sum_{t=0}^{T-1} p_t\T l_t - \!\!\!\!\min_{i\in \{1,\dots,m\}} \!\sum_{t=0}^{T-1} l_t^i &\leq L \sqrt{\text{log}(m) T/2}+4L \text{log}(m).
\end{align*}
We note that all of our steps, except for the application of Popoviciu's inequality, are tight for small $\epsilon$, i.e., for a large time horizon $T$. It is therefore not surprising that there is a matching lower bound, which indeed holds for large $T$ and $m$. Moreover, the additional constant $4L \text{log}(m)$ can be avoided with a more careful bookkeeping. We summarize the sharpened results in the next proposition.
\begin{proposition}\label{Prop:MW} (See, e.g., \citet[Ch.~18]{GTalive})
The regret of Alg.~\ref{Alg:MW} is bounded by
\begin{equation*}
\sum_{t=0}^{T-1} p_t\T l_t - \!\!\!\!\min_{i\in \{1,\dots,m\}} \!\sum_{t=0}^{T-1} l_t^i \leq L \sqrt{\text{log}(m) T/2}.
\end{equation*}
Moreover, for any decision algorithm and constant $\epsilon'>0$ there exists $T>0$, $m>0$ and a deterministic loss sequence $l_0,l_1,\dots,$ such that
\begin{equation*}
\sum_{t=0}^{T-1} p_t\T l_t - \!\!\!\!\min_{i\in \{1,\dots,m\}} \!\sum_{t=0}^{T-1} l_t^i \geq L \sqrt{\text{log}(m) T/2} - \epsilon'.
\end{equation*}
\end{proposition}

\section{Decision-making with constraints}\label{Sec:DMC}
As we have seen in the previous section, the analysis of multiplicative weights hinges on a Taylor series expansion in \eqref{eq:rec1} and an appropriate choice of $p_t$ that results in vanishing first-order terms. However, we also realize that the final bound arises from the sum of all the second-order terms. This motivates our research, which develops an update scheme that factors in constraints and improves the second-order terms.

Our treatment starts by parametrizing $p_t$ in the following way
\begin{equation}
p_t= \frac{w_{t-1}}{\phi_{t-1}} - \frac{\epsilon}{2} \underbrace{\left( \frac{\text{diag}(w_{t-1})}{\phi_{t-1}} - \frac{w_{t-1} w_{t-1}\T}{\phi_{t-1}^2} \right)}_{:=Q_{t-1}} q_t, \label{eq:param}
\end{equation}
where $q_t \in \mathbb{R}^m$ will be the decision-variable of our algorithm. The parameterization is designed in such a way that $\sum_{i=1}^{m} p_t^i=1$ for all $q_t\in \mathbb{R}^m$ and recovers multiplicative weights for $q_t=0$. We also note that $Q_{t-1}$ is positive semi-definite (see App.~\ref{App:C}). As a consequence of the parametrization, we obtain the following expression for the remainder term part ii in \eqref{eq:importantEq}
\begin{equation*}
\text{part~ii}=1 + \frac{\epsilon^2}{2} l_t\T Q_{t-1} (l_t-q_t) + \frac{\epsilon^4}{4} (l_t\T Q_{t-1} q_t)^2 + \frac{\epsilon^3}{6} \rho_t.
\end{equation*}
We therefore conclude that the second-order term can be actively controlled, for example by choosing $q_t$ as follows
\begin{align}
r_t^*:=\min_{q_t\in \mathcal{Q}_{t}}&\max_{l_t \in X_t} l_t\T Q_{t-1} (l_t-q_t), \label{eq:prob1}
\end{align}
where the set $\mathcal{Q}_t$ is defined as
\begin{equation*}
\mathcal{Q}_t:=\Big\{ q \in \mathbb{R}^m ~|~q\T \mathbf{1} = 0, ~\mathbf{1} \geq  \frac{\epsilon}{2} (q - \mathbf{1} w_{t-1}\T q/\phi_{t-1})\Big\},
\end{equation*}
with $\mathbf{1}:=(1,\dots,1)\T\in \mathbb{R}^m$. The set imposes $p_t\geq 0$ and includes a normalization with $q\T \mathbf{1}=0$, since both \eqref{eq:param} and \eqref{eq:prob1} are invariant to transformations of the type $q' \gets q + \delta \mathbf{1}$. We can express the objective function in \eqref{eq:prob1} more succinctly by following elementary manipulations:
\begin{align*}
\sum_{i,j=1}^{m} \!\frac{w_{t-1}^i w_{t-1}^j}{2\phi_{t-1}^2}\!\left( (l_t^i-l_t^j)^2 \!- \!(l_t^i-l_t^j) (q_t^i-q_t^j)\right).
\end{align*}
We note that \eqref{eq:prob1} amounts to a somewhat pessimistic perspective, where the agent considers the worst case over all the possible choices $l_t\in X_t$ of the environment. However, compared to multiplicative weights, we do not assume that all weights $w_{t-1}^i$, $i=1,\dots,m$ are equal (this done when applying Popoviciu's inequality), which allows Alg.~\ref{Alg:EMW} and our analysis to adapt and exploit an environment that plays in a suboptimal manner. Furthermore, the constraints $X_t$ are allowed to depend on past actions $a_{t-1}, a_{t-2}, \dots$ and similarly on past losses $l_{t-1},l_{t-2}, \dots$, and can therefore also account for budget constraints on the losses $l_t$. 

The remaining part of this section establishes a regret bound for Alg.~\ref{Alg:EMW}, which is based on choosing $q_t$ according to \eqref{eq:prob1}.

\begin{algorithm}
\caption{Multiplicative weights with constraints}
\begin{algorithmic}
\Require $c_1>0$, $c_2>0$ \Comment{e.g. $c_1=10^{-2}$, $c_2=L^2/4$}
\State $\bar{r}^* \gets c_1 \log(m)^{2/3} L^2 T^{-1/3}$
\For{$t=0~\mathrm{to}~T-1$}
	\State $\epsilon_{t-1} \gets \sqrt{\frac{2 \text{log}(m)}{c_2+\sum_{j=0}^{t-1} \tilde{r}_j^*}}$ \Comment{$\epsilon_{-1}=\sqrt{2 \text{log}(m)/c_2}$}
	\State $w_{t-1}^i \gets \exp(-\epsilon_{t-1} \sum_{j=0}^{t-1} l_j^i)$ \Comment{$w_{-1}^i=1$}
	\State $\phi_{t-1} \gets \sum_{i=1}^m w_{t-1}^i$
	\State $q_t^* \gets \argmin_{q_t\in \mathcal{Q}_t} \max_{l_t\in X_t} l_t\T Q_{t-1} (l_t-q_t)$
	\State $p_t \gets \frac{w_{t-1}}{\phi_{t-1}} - \frac{\epsilon_{t-1}}{2} Q_{t-1} q_{t}^*$
	\State Sample $a_t$ according to $p_t$, observe $l_t$
	\State $\tilde{r}_t^* \gets \max\{ l_t\T Q_{t-1} (l_t-q_t^*),\bar{r}^* \}$
\EndFor
\end{algorithmic}
\label{Alg:EMW}
\end{algorithm}
We note that Alg.~\ref{Alg:EMW} adapts the parameter $\epsilon_{t-1}$ on the go. (This is due to the fact that $r_t^*$ depends on $w_t^i$, which is only available at time $t$.) In order to simplify the exposition, we start by analyzing the constant step-size version of Alg.~\ref{Alg:EMW}; it turns out that the adaptive algorithm can be analyzed with the same ideas, see App.~\ref{App:proof}. Furthermore, in our analysis the constant $L$ will be defined as 
\begin{equation*}
L:=\max_{t\in \{0,1,\dots,T-1\}} \max_{l_t\in X_t} |l_t^i-l_t^j|.
\end{equation*}
According to the previous discussion, we obtain the following bound for part ii:
\begin{equation*}
\text{part~ii} \leq 1 + \epsilon^2 r_t^*/2 + \epsilon^3 L^3/2,
\end{equation*}
where $r_t^*$ is given by \eqref{eq:prob1} and we have used the fact that $q_t\T Q_{t-1} l_t \leq L^2/4$ and that $\rho_t$ (see \eqref{eq:importantEq}) is bounded by $L^3 \exp(1)$. We note that this matches almost verbatim the derivation of multiplicative weights. Thus, by following the same steps, which include unrolling the recursion for $\phi_t$, rearranging terms and dividing by $\epsilon$ we obtain
\begin{equation*}
\sum_{t=0}^{T-1} p_t\T l_t - \!\!\!\!\min_{i\in \{1,\dots,m\}} \!\sum_{t=0}^{T-1} l_t^i \leq \text{log}(m)/\epsilon + \epsilon \sum_{t=0}^{T-1} r_t^*/2 + \epsilon^2 L^3 T/2. 
\end{equation*}
By choosing the parameter $\epsilon$ to be $\sqrt{2 \text{log}(m)/\sum_{j=0}^{T-1} r_j^*}$, we finally arrive at the following bound for the regret
\begin{equation*}
\sum_{t=0}^{T-1} p_t\T l_t - \!\!\!\!\min_{i\in \{1,\dots,m\}} \!\sum_{t=0}^{T-1} l_t^i \leq \sqrt{2 \text{log}(m) \sum_{j=0}^{T-1} r_j^*} + \frac{\text{log}(m) L^3}{\bar{r}^*},
\end{equation*}
where $\bar{r}^*>0$ denotes a lower bound on $r_j^*$.\footnote{We assume $r_j^*\geq 0$ without loss of generality. If $\sum_{j=0}^{T-1} r_j^*= \mathcal{O}(T^{2/3})$, then an even smaller regret on the order $T^{1/3}$ can be achieved. It is clear that a slightly improved choice of $\epsilon$ can balance the two cases, i.e., transition from a $\mathcal{O}(\sqrt{T})$ to a $\mathcal{O}(T^{1/3})$ regret. This is included in Alg.~\ref{Alg:EMW}, but we have omitted further explanations to simplify the presentation.} The first term on the right-hand side, closely resembles the regret bound from the previous section. In fact, for $X_t=[0,L]\times \dots \times [0,L]$, $r_t^*$ can be bounded by $L^2/4$, which recovers the result from the previous section. However, if the accumulation of $r_j^*$ is small, the algorithm proposed in this section largely outperforms multiplicative weights. Alg.~\ref{Alg:EMW}, which chooses the parameter $\epsilon$ adaptively, achieves almost the same result, as summarized by the following proposition. The proof is a bit more involved due to the fact that $\epsilon$ is varying, see App.~\ref{App:proof}.
\begin{proposition}\label{Prop:RE}
Let $c_2=2 \log(m) L^2$. Then, the regret of Alg.~\ref{Alg:EMW} is bounded by
\begin{multline*}
\sum_{t=0}^{T-1} p_t\T l_t - \!\!\!\!\min_{i\in \{1,\dots,m\}} \!\sum_{t=0}^{T-1} l_t^i \leq 3 \sqrt{ \text{log}(m) \sum_{j=0}^{T-1} \tilde{r}_j^*}\\
+ \frac{7 \log(m) L^3}{4 \bar{r}^*} \text{log}\left(8 \log(m) + 4/L^2 \sum_{j=0}^{T-1} \tilde{r}_j^*\right).
\end{multline*}
\end{proposition}

\section{Intervals}\label{Sec:Intervals}
In the previous section $X_t$ was allowed to be an arbitrary closed subset of $\mathbb{R}^m$. We now turn to the case where $X_t=[\underline{l}_t^1,\bar{l}_t^1]  \times \dots \times [\underline{l}_t^m,\bar{l}_t^m]$, that is, where each loss $l_t^i$ is restricted to the interval $[\underline{l}_t^i,\bar{l}_t^i]$. This leads to an important simplification in \eqref{eq:prob1}, which is summarized by the following lemma.
\begin{lemma}\label{Lemma:Attmax}
Let $q_t\in \mathbb{R}^m$ be an arbitrary vector, and $X_t=[\underline{l}_t^1,\bar{l}_t^1]  \times \dots \times [\underline{l}_t^m,\bar{l}_t^m]$. Then $l_t\T Q_{t-1} (l_t-q_t)$ achieves its maximum for $l_t\in X_t$ on the corners of $X_t$, that is, when $l_t\in \{\underline{l}_t^1, \bar{l}_t^1\} \times \dots \times \{\underline{l}_t^m,\bar{l}_t^m\}$.
\end{lemma}
\begin{proof}
The proof is by contradiction and starts by assuming that the maximizer $l_t^*\in X_t$ has at least one component $j$, which is in the interior of the interval $[\underline{l}_t^j,\bar{l}_t^j]$. We now consider $l_t=l_t^*+e_j \delta$, where $e_j\in \mathbb{R}^m$ is the $j$th unit vector of the standard basis in $\mathbb{R}^m$, and $\delta \in \mathbb{R}$ is a small scalar. For small values of $\delta$, we are guaranteed that $l_t\in X_t$. Moreover, evaluating the objective function at $l_t$ yields
\begin{equation*}
(l_t^*)\T Q_{t-1} (l_t^*-q_t^*) + \delta e_j\T Q_{t-1} (2 l_t^*-q_t) + \delta^2 e_j\T Q_{t-1} e_j,
\end{equation*}
where $e_j\T Q_{t-1} e_j$ is strictly positive due to the positivity of the weights $w_{t-1}$. The term linear in $\delta$ is either nonnegative or strictly negative. Hence, by choosing a small enough $\delta$, either $\delta>0$ or $\delta <0$, the above expression exceeds the optimal function value $(l_t^*)\T Q_{t-1} (l_t^*-q_t^*)$, which leads to a contradiction.\hfill\qed
\end{proof}
Lemma~\ref{Lemma:Attmax} leads to the following important conclusions, as summarized by the next proposition.

\begin{proposition}\label{Prop:main}
Let $E_t:=\{\underline{l}_t^1,\bar{l}_t^1\}  \times \dots \times \{\underline{l}_t^m,\bar{l}_t^m\}$. Then, the minmax problem stated in \eqref{eq:prob1} is equivalent to the linear program
\begin{multline*}
r_t^* = \min_{q_t \in \mathcal{Q}_t, \xi\in \mathbb{R}} \xi \quad \text{s.t.}\quad \xi \geq l_t\T Q_{t-1} (l_t-q_t), ~\forall l_t\in E_t,
\end{multline*}
which includes $2^m+m+1$ linear constraints. Moreover, the resulting $q_t^*$ is optimal in the following sense: there exists a randomized strategy $l_t^*$ of the environment, such that
\begin{equation*}
\text{E}_{l_t^*}[ {l_t^*}\T Q_{t-1} (l_t^*-q)]\geq r_t^*, \quad \forall q\in \mathcal{Q}_t.
\end{equation*}
\end{proposition}
\begin{proof}
The first part is an immediate corollary of Lemma~\ref{Lemma:Attmax}. For the second part (optimality) we consider any randomized strategy $l_t$ that only selects elements in $E_t$. As a result, in expectation, the cost is given by
\begin{equation*}
\sum_{l\in E_t} l\T Q_{t-1} (l-q) \text{Pr}(l_t=l),
\end{equation*}
which is a linear function in $q$ and $\{\text{Pr}(l_t=l)\}_{l\in E_t}$. Moreover, $\mathcal{Q}_t$ is compact and convex, which yields, according to von Neumann's minimax theorem
\begin{multline}
\max_{\{\text{Pr}(l_t=l)\}_{l\in E_t}} \min_{q_t\in \mathcal{Q}_t} \sum_{l\in E_t} l\T Q_{t-1} (l-q_t) \text{Pr}(l_t=l) = \\
\min_{q_t\in \mathcal{Q}_t} \max_{\{\text{Pr}(l_t=l)\}_{l\in E_t}} \sum_{l\in E_t} l\T Q_{t-1} (l-q_t) \text{Pr}(l_t=l).\label{eq:minimax}
\end{multline}
We consider the maximum on the right-hand side: For a fixed $q_t$, the objective is linear in $\text{Pr}(l_t=l)$, which means that the maximum is attained by choosing $l_t\in E_t$ deterministically. In other words, the right-hand side achieves the same cost as
\begin{equation*}
\min_{q_t\in \mathcal{Q}} \max_{l_t\in E_t} l_t\T Q_{t-1} (l_t-q_t)=r_t^*.
\end{equation*}
Moreover, the maximum on the left-hand side of \eqref{eq:minimax} is guaranteed to be attained for some probability distribution $\{\text{Pr}(l_t^*=l)\}_{l\in E_t}$, which implies 
\begin{equation*}
\min_{q_t\in \mathcal{Q}_t} \sum_{l\in E_t} l\T Q_{t-1}(l-q_t) \text{Pr}(l_t^*=l) = r_t^*,
\end{equation*}
leading to the desired conclusion. \hfill \qed
\end{proof}
Prop.~\ref{Prop:main} provides three important results: i) $q_t$ can be obtained by solving a linear program; ii) the resulting $q_t$ is optimal, in the sense that there is no algorithm that is able to improve upon \eqref{eq:importantEq} for small enough $\epsilon$. We note that the second order terms in \eqref{eq:importantEq} determine the regret up to a constant (see Sec.~\ref{Sec:DMC}); iii) the minimax theorem gives us an explicit approach for computing the worst-case strategy $l_t^*$ of the environment. For a given sequence of $X_t$, we can therefore recursively compute the strategies $l_t^*$ for which any algorithm suffers at least the loss $r_t^*$, $t=0,2,\dots,T-1$ in expectation. Compared to the lower bound established in Prop.~\ref{Prop:MW}, this allows for a much more fine-grained analysis and points to the fact that the result in Prop.~\ref{Prop:RE} cannot be improved except for constant (additive and multiplicative) factors. 

Unfortunately, the linear program stated in Prop.~\ref{Prop:main} involves $2^m + m +1$ constraints. If $m$ is large, its solution can be computationally challenging. The next proposition provides a simple strategy for choosing $p_t$ when $m$ is large. The strategy is even optimal for $m=2$ (see next section) and also reduces to multiplicative weights, whenever $\bar{l}_t^1=\dots=\bar{l}_t^m$ and $\underline{l}_t^1=\dots=\underline{l}_t^m$ with $\bar{l}_t^i-\underline{l}_t^i=L$.
\begin{proposition}\label{Prop:Approx}
Without loss of generality we assume that all intervals overlap pairwise.\footnote{If the intervals do not overlap pairwise, we can easily exclude the intervals $i$, which are such that $\underline{l}_t^i\geq \bar{l}_t^j$ for some $j\in \{1,\dots,m\}$. This is done by setting the corresponding $p_t^i=0$ and $w_{t-1}^i=0$. This will also be discussed in the next section (see case $B$).} An approximate solution to \eqref{eq:prob1} is given by
\begin{equation*}
\hat{q}_t = \argmin_{q\in \mathcal{Q}_t} \sum_{i=1}^{m} \Big|q^i- \mu_t^i\Big|^2, ~\mu_t^i:=\bar{l}_t^i + \underline{l}_t^i - \frac{1}{m} \sum_{j=1}^{m} (\bar{l}_t^j+\underline{l}_t^j).
\end{equation*}
For small enough $\epsilon$, the approximate solution achieves 
\begin{align*}
\max_{l_t\in X_t} l_t\T Q_{t-1} (l_t-\hat{q}_t) &\leq \sum_{i,j=1, i\neq j}^{m}\frac{w_{t-1}^i w_{t-1}^j}{2\phi_{t-1}^2} (\bar{l}_t^i - \underline{l}_t^j) (\bar{l}_t^j-\underline{l}_t^i).
\end{align*}
\end{proposition}
We note that the proposition chooses each $q_t^i$ as close as possible to $\mu_t^i$ subject to the constraint $\mathcal{Q}_t$. The value $\mu_t^i$ compares the mid point of the interval $i$ to the average of all mid points. 
The proposition will be further motivated in the next section and a proof is included in App.~\ref{App:B}.

\subsection{Closed-form solutions for $m=2$}
For $m=2$ we can compute the solution of \eqref{eq:prob1} in closed form. According to Lemma~\ref{Lemma:Attmax} we are guaranteed that the maximum over $l_t$ is attained for $l_t^1\in \{\underline{l}_t^1,\bar{l}_t^1\}$ and $l_t^2\in \{\underline{l}_t^2,\bar{l}_t^2\}$, and we can also rule out the cases $l_t^1=\underline{l}_t^1$, $l_t^2=\underline{l}_t^2$ and $l_t^1=\bar{l}_t^1$, $l_t^2=\bar{l}_t^2$. This leaves us with
\begin{multline*}
\min_{q_t\in \mathcal{Q}_t} \frac{w_{t-1}^1 w_{t-1}^2}{\phi_{t-1}^2} \max\{ (\bar{l}_t^{1}-\underline{l}_t^{2})^2 - (\bar{l}_t^1-\underline{l}_t^2) (q_t^1 - q_t^2), \\
(\underline{l}_t^1 - \bar{l}_t^2)^2 - (\underline{l}_t^1-\bar{l}_t^2) (q_t^1-q_t^2)\}.
\end{multline*}
The objective is therefore given by the maximum of two linear functions. More importantly, the weights $w_{t-1}$ merely affect the scaling of the objective function and are therefore irrelevant to the solution. This is no longer true for $m>2$. Depending on the values of $\bar{l}_t^1-\underline{l}_t^2$ and $\bar{l}_t^2-\underline{l}_t^1$ we distinguish three different cases, as summarized in Fig.~\ref{Fig:sketch}. (In total there are six different cases, but the remaining three can be reduced to $A$, $B$, $C$ by relabeling the losses $l_t^i$.) We note that in case $B$ the objective is unbounded below, which means that the resulting optimal strategy is given by $p_t^1=0$, $p_t^2=1$. This clearly matches our intuition, since the agent has no incentive in playing option $1$. In case $A$ and $C$ the optimal solution for \eqref{eq:prob1} is given by
\begin{equation*}
q_t^1=\begin{cases} \mu_t^{1}, & \text{for}~-\frac{1}{\epsilon} \frac{\phi_{t-1}}{w_{t-1}^1} \leq \mu_t^{1} \leq \frac{1}{\epsilon} \frac{\phi_{t-1}}{w_{t-1}^2}\\
\frac{1}{\epsilon} \frac{\phi_{t-1}}{w_{t-1}^2}, &\text{for}~\mu_t^{1} \geq \frac{1}{\epsilon} \frac{\phi_{t-1}}{w_{t-1}^2}\\
-\frac{1}{\epsilon} \frac{\phi_{t-1}}{w_{t-1}^1}, &\text{for}~\mu_t^{1} \leq -\frac{1}{\epsilon} \frac{\phi_{t-1}}{w_{t-1}^1},
\end{cases}
\end{equation*}
where $\mu_t^{1}$ is defined in Prop.~\ref{Prop:Approx}. By definition, we have $q_t^2=-q_t^1$. The previous equation has a very intuitive meaning, since $\mu_t^{1}$ is nothing but the difference between the midpoint of the interval $[\underline{l}_t^1, \bar{l}_t^1]$ and the midpoint of the interval $[\underline{l}_t^2,\bar{l}_t^2]$.

Provided that $q_t^1=\mu_t^1$ and $q_t^2=\mu_t^2$ the resulting cost is given by
\begin{equation*}
\frac{w_{t-1}^1 w_{t-1}^2}{\phi_{t-1}^2} (\bar{l}_t^1-\underline{l}_t^2) (\bar{l}_t^2-\underline{l}_t^1);
\end{equation*}
a lower cost is attained in the remaining cases for $\epsilon$ small enough. The term $(\bar{l}_t^1-\underline{l}_t^2) (\bar{l}_t^2-\underline{l}_t^1)$ can be reformulated as
\begin{equation*}
\underbrace{\left(\frac{\bar{l}_t^1-\underline{l}_t^1 + \bar{l}_t^2-\underline{l}_t^2}{2}\right)^2}_{\text{average~length}^2} - \underbrace{\left( \frac{\bar{l}_t^1 +\underline{l}_t^1}{2} - \frac{\bar{l}_t^2+\underline{l}_t^2}{2}\right)^2}_{\text{distance~between~center}^2},
\end{equation*}
which implies that for two intervals of fixed length, the worst-case cost is maximal if the intervals are centered and vanishes when $\bar{l}_t^1=\underline{l}_t^2$ or $\underline{l}_t^1=\bar{l}_t^2$. It also directly reduces to $L^2/4$ in the multiplicative weights setting and therefore generalizes the analysis from Sec.~\ref{Sec:MW}.

\begin{figure}
\center
\def\svgwidth{.95\columnwidth}
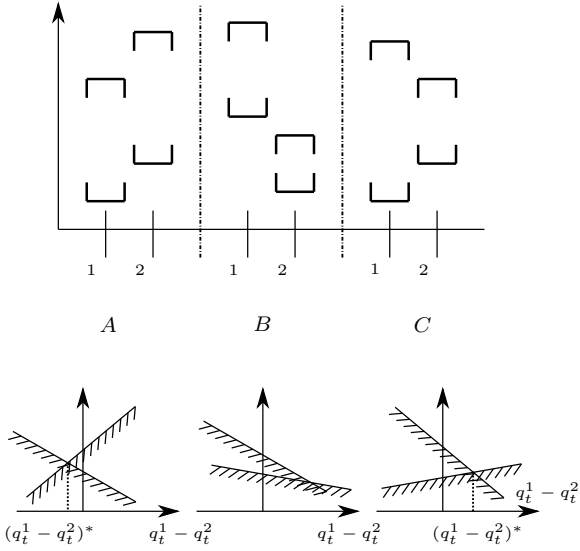
\caption{This figure illustrates the three different cases, where $A$ is given by $\bar{l}_t^1-\underline{l}_t^2\geq \bar{l}_t^1-\underline{l}_t^2\geq 0$, $B$ by $\bar{l}_t^1\geq \underline{l}_t^2$, and $C$ by $\bar{l}_t^1-\underline{l}_t^2\geq 0$, $\underline{l}_t^2- \bar{l}_t^2\leq 0$. In case $A$ the minimum over $q_t^1-q_t^2$ is attained for negative values, whereas for case $C$ the minimum is attained for positive values. For case $B$ the objective function is unbounded, which matches our intuition since the interval $[\underline{l}_t^1,\bar{l}_t^1]$ dominates $[\underline{l}_t^2,\bar{l}_t^2]$.}
\label{Fig:sketch}
\end{figure}

\section{Numerical examples}\label{Sec:Numerics}
The following section contains two numerical examples that highlight the effectiveness of Alg.~\ref{Alg:EMW}.

\subsection{Random intervals}
In the first example we choose $X_t=[\underline{l}_t^1,\bar{l}_t^1]\times \dots [\underline{l}_t^m,\bar{l}_t^m]$, where each $\underline{l}_t^i$ and $\bar{l}_t^i$ correspond to two independent samples from the uniform distribution over $[0,1]$. (The samples are ordered such that $\underline{l}_t^i \leq \bar{l}_t^i$.) The losses $l_t^i$ are obtained by sampling from the uniform distribution over $[\underline{l}_t^i, \bar{l}_t^i]$. Fig.~\ref{Fig:RI1} and Fig.~\ref{Fig:RI2} compare Alg.~\ref{Alg:EMW} to multiplicative weights for $T=200$ and $m=10$, and $T=200$, $m=100$. In the latter case the optimization \eqref{eq:prob1} is approximated with the method summarized in Prop.~\ref{Prop:Approx}. We note that Alg.~\ref{Alg:EMW} indeed exploits the available knowledge about the intervals $[\underline{l}_t^i,\bar{l}_t^i]$ leading to an accumulated cost that is substantially below the cost of the best single option in hindsight (negative regret).

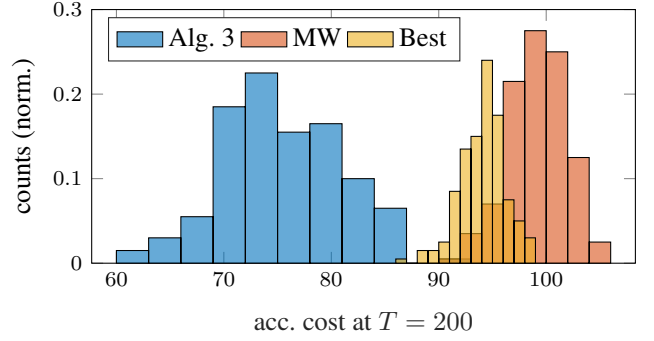
\begin{figure}
\newlength\figureheight 
\newlength\figurewidth
\setlength{\figureheight}{.38\columnwidth}
\setlength{\figurewidth}{.85\columnwidth}
%
%
\definecolor{mycolor1}{rgb}{0.00000,0.44700,0.74100}%
\definecolor{mycolor2}{rgb}{0.85000,0.32500,0.09800}%
\definecolor{mycolor3}{rgb}{0.92900,0.69400,0.12500}%
\begin{tikzpicture}

\begin{axis}[%
width=0.951\figurewidth,
height=\figureheight,
at={(0\figurewidth,0\figureheight)},
scale only axis,
xmin=57.7,
xmax=108.3,
xlabel style={font=\color{white!15!black}},
xlabel={acc. cost at $T=200$},
ymin=0,
ymax=0.3,
ylabel style={font=\color{white!15!black}},
ylabel={counts (norm.)},
ylabel near ticks,
axis background/.style={fill=white},
legend columns=-1,
legend style={legend cell align=left, align=left, draw=white!15!black},
legend pos=north west
]
\addplot[ybar interval, fill=mycolor1, fill opacity=0.6, draw=black, area legend] table[row sep=crcr] {%
x	y\\
60	0.015\\
63	0.03\\
66	0.055\\
69	0.185\\
72	0.225\\
75	0.155\\
78	0.165\\
81	0.1\\
84	0.065\\
87	0\\
90	0.005\\
93	0.005\\
};
\addlegendentry{Alg.~3}

\addplot[ybar interval, fill=mycolor2, fill opacity=0.6, draw=black, area legend] table[row sep=crcr] {%
x	y\\
90	0.005\\
92	0.035\\
94	0.07\\
96	0.215\\
98	0.275\\
100	0.25\\
102	0.125\\
104	0.025\\
106	0.025\\
};
\addlegendentry{MW}

\addplot[ybar interval, fill=mycolor3, fill opacity=0.6, draw=black, area legend] table[row sep=crcr] {%
x	y\\
86	0.005\\
87	0\\
88	0.015\\
89	0.015\\
90	0.025\\
91	0.085\\
92	0.135\\
93	0.15\\
94	0.24\\
95	0.175\\
96	0.075\\
97	0.05\\
98	0.03\\
99	0.03\\
};
\addlegendentry{Best}

\end{axis}
\end{tikzpicture}%
\caption{The figure shows the histogram of the accumulated cost for problems where the intervals $[\underline{l}_t^i,\bar{l}_t^i]$ are randomly generated in $[0,1]$. We set $m=10$ and $T=200$, and directly solve \eqref{eq:prob1}. ``Best" denotes the best option in hindsight and ``MW" multiplicative weights. We note that Alg.~\ref{Alg:EMW} largely outperforms the multiplicative weights algorithm.} \label{Fig:RI1}
\end{figure}

\begin{figure}
\setlength{\figureheight}{.38\columnwidth}
\setlength{\figurewidth}{.85\columnwidth}
%
%
\definecolor{mycolor1}{rgb}{0.00000,0.44700,0.74100}%
\definecolor{mycolor2}{rgb}{0.85000,0.32500,0.09800}%
\definecolor{mycolor3}{rgb}{0.92900,0.69400,0.12500}%
\begin{tikzpicture}

\begin{axis}[%
width=0.951\figurewidth,
height=\figureheight,
at={(0\figurewidth,0\figureheight)},
scale only axis,
xmin=47.1,
xmax=110.9,
xlabel style={font=\color{white!15!black}},
xlabel={acc. cost at $T=200$},
ymin=0,
ymax=0.4,
ylabel style={font=\color{white!15!black}},
ylabel={counts (norm.)},
ylabel near ticks,
axis background/.style={fill=white},
legend columns=-1,
legend style={legend cell align=left, align=left, draw=white!15!black},
legend pos=north west
]
\addplot[ybar interval, fill=mycolor1, fill opacity=0.6, draw=black, area legend] table[row sep=crcr] {%
x	y\\
50	0.015\\
52	0.085\\
54	0.18\\
56	0.26\\
58	0.25\\
60	0.155\\
62	0.05\\
64	0.005\\
66	0.005\\
};
\addlegendentry{Alg.~3}

\addplot[ybar interval, fill=mycolor2, fill opacity=0.6, draw=black, area legend] table[row sep=crcr] {%
x	y\\
90	0.005\\
92	0.03\\
94	0.075\\
96	0.13\\
98	0.28\\
100	0.335\\
102	0.125\\
104	0.015\\
106	0.005\\
108	0.005\\
};
\addlegendentry{MW}

\addplot[ybar interval, fill=mycolor3, fill opacity=0.6, draw=black, area legend] table[row sep=crcr] {%
x	y\\
86	0.005\\
87	0.03\\
88	0.045\\
89	0.1\\
90	0.215\\
91	0.365\\
92	0.175\\
93	0.06\\
94	0.005\\
95	0.005\\
};
\addlegendentry{Best}

\end{axis}
\end{tikzpicture}%
\caption{The figure shows the histogram of the accumulated cost for problems where the intervals $[\underline{l}_t^i,\bar{l}^i]$ are randomly generated in $[0,1]$. We set $m=100$ and $T=200$, and we have chosen $\epsilon$ as specified in Alg.~\ref{Alg:EMW} and Alg.~\ref{Alg:MW} with $L=1$. The optimization in Alg.~\ref{Alg:EMW} is approximated with Prop.~\ref{Prop:Approx}. ``Best" denotes the best option in hindsight and ``MW" multiplicative weights. We note that Alg.~\ref{Alg:EMW} largely outperforms the multiplicative weights algorithm.}\label{Fig:RI2}
\end{figure}
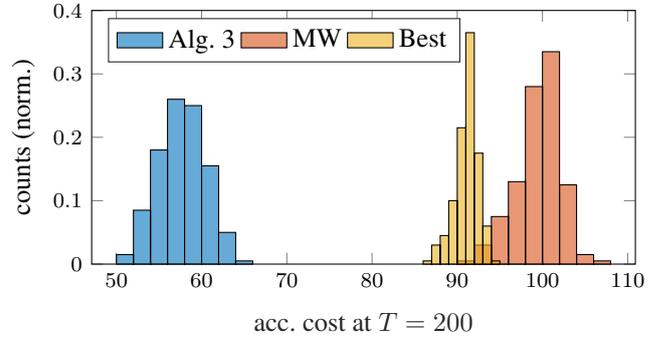

\subsection{Nonlinear online identification}
The decision-making framework discussed in this article can be used for nonlinear and online system identification, which we will illustrate next. We assume that there exists a finite number of candidate models, for example
\begin{equation}
x_{t+1}=\theta x_t (1-x_t), \quad \theta \in \{\theta^1,\dots,\theta^m\}, \label{eq:logisticmap}
\end{equation}
where $x_t \in \mathbb{R}$ denotes the state at time $t$, and where one of the candidate models accurately describes the behavior of the underlying unknown system. The losses $l_t^i$ are chosen as the one-step prediction error,
\begin{equation*}
l_t^i=|\hat{x}_{t+1}-\theta^i \hat{x}_t (1-\hat{x}_t)|,
\end{equation*}
where $\hat{x}_{t+1}$ denotes the state of the underlying system at time $t+1$. The dynamics specified in \eqref{eq:logisticmap} are known to exhibit chaos, depending on the value of $\theta$.
 We note that our results have a logarithmic dependency on the number of options $m$. This means that if $\theta$ arises from a discretization of a bounded region in $n_p$ dimensions with width $\delta$, $\log(m) \sim n_p \log(1/\delta)$ and hence the regret is expected to scale with $\sqrt{n_p}$, which is relatively mild. In our numerical example we set $\theta^i$ to be an equidistant discretization with $m=50$ of the interval $[3.0,3.9]$. Moreover, the underlying dynamics were set to
\begin{equation*}
\hat{x}_{t+1}=(3.57+n_t) \hat{x}_t (1-\hat{x}_t),\quad n_t \sim \text{Unif}([-0.05,0.05]),
\end{equation*}
where the noise samples $n_t$ are independent. We note that the value of 3.57 exhibits chaos and lies between $\theta^{32}$ and $\theta^{33}$. The intervals $\underline{l}_t^i$ and $\bar{l}_t^i$ are obtained by maximizing, respectively minimizing $|(3.57+n-\theta^i) \hat{x}_t (1-\hat{x}_t)|$ for $n\in [-0.05,0.05]$. The evolution of the resulting accumulated cost is shown in Fig.~\ref{Fig:SisID} and we have again approximated \eqref{eq:prob1} with Prop.~\ref{Prop:Approx}. The plot indicates that Alg.~\ref{Alg:EMW} has again a competitive edge on multiplicative weights, even though for this example, the difference is less pronounced. We believe that this is due to two factors: i) the information contained in the intervals $[\underline{l}_t^i,\bar{l}_t^i]$ is limited, ii) the environment, which chooses the losses is relatively benign in the sense that it closely follows the model given by $\theta^{32}$ and/or $\theta^{33}$.

\begin{figure}
\setlength{\figureheight}{.4\columnwidth}
\setlength{\figurewidth}{.85\columnwidth}
%
%
\begin{tikzpicture}

\begin{axis}[%
width=0.951\figurewidth,
height=\figureheight,
at={(0\figurewidth,0\figureheight)},
scale only axis,
xmin=0,
xmax=50,
xlabel style={font=\color{white!15!black}},
xlabel={iterations $t$},
ymin=0,
ymax=0.25,
ylabel style={font=\color{white!15!black}},
ylabel={acc. prediction error},
axis background/.style={fill=white},
legend columns=-1,
legend style={legend cell align=left, align=left, draw=white!15!black},
legend pos=north west
]
\addplot [color=black]
  table[row sep=crcr]{%
1	0\\
2	0.0357666019479123\\
3	0.0516962660890816\\
4	0.0621989578214326\\
5	0.0692969415605463\\
6	0.0759362260489359\\
7	0.0804619870570066\\
8	0.0859937938672743\\
9	0.0896503019263637\\
10	0.0959916572252258\\
11	0.0973970079324312\\
12	0.102144297278968\\
13	0.104522794643486\\
14	0.110237226928442\\
15	0.111851781851925\\
16	0.11724945880957\\
17	0.119814464895712\\
18	0.122907683206077\\
19	0.12431996883003\\
20	0.126862941071038\\
21	0.130029280579328\\
22	0.13259886718288\\
23	0.133984531434706\\
24	0.138210968389674\\
25	0.139447312112662\\
26	0.146136946000155\\
27	0.147545961728444\\
28	0.149430993149438\\
29	0.151260318442964\\
30	0.154015731561189\\
31	0.156064722351564\\
32	0.160214724474434\\
33	0.161733328713557\\
34	0.16525381484267\\
35	0.167321461849098\\
36	0.169357231740467\\
37	0.17163881506335\\
38	0.176011359520403\\
39	0.178199092978793\\
40	0.179500178678664\\
41	0.180387519214708\\
42	0.182946621420452\\
43	0.183317690042109\\
44	0.189021840574216\\
45	0.189522358080861\\
46	0.191887805374676\\
47	0.192204031472932\\
48	0.195834690416312\\
49	0.199482476254401\\
50	0.202174752499527\\
};
\addlegendentry{Alg.~3}

\addplot [color=red]
  table[row sep=crcr]{%
1	0\\
2	0.0508631505991063\\
3	0.0687231300020937\\
4	0.0798170041614423\\
5	0.0866630662488148\\
6	0.0930753872011186\\
7	0.0975063369429313\\
8	0.102727093587912\\
9	0.10650938770133\\
10	0.112046822557388\\
11	0.113449921667697\\
12	0.117689799695642\\
13	0.119978109261367\\
14	0.12589430507597\\
15	0.12747826259424\\
16	0.132505758194642\\
17	0.134996152911904\\
18	0.137809613137537\\
19	0.139192540452977\\
20	0.141553959876643\\
21	0.144643292867934\\
22	0.147018126531549\\
23	0.148375363937742\\
24	0.152564326014835\\
25	0.153808489139973\\
26	0.160469753486448\\
27	0.161897489284458\\
28	0.163447855856881\\
29	0.165251599760012\\
30	0.167837139918698\\
31	0.169904615640616\\
32	0.17391812401211\\
33	0.175402408777453\\
34	0.17865688844595\\
35	0.180744150064583\\
36	0.182513231867447\\
37	0.184814957812667\\
38	0.189261654212986\\
39	0.191432679676571\\
40	0.192405789134335\\
41	0.193296727704896\\
42	0.195563160840514\\
43	0.19593909641306\\
44	0.201544387054982\\
45	0.202028908168581\\
46	0.204284851006117\\
47	0.204604323929013\\
48	0.208285130416763\\
49	0.21194012937159\\
50	0.214447011605748\\
};
\addlegendentry{MW}

\addplot [color=black, dashed]
  table[row sep=crcr]{%
1	0\\
2	8.75009458118914e-05\\
3	0.000851211347869807\\
4	0.00518365682349398\\
5	0.00782989264191669\\
6	0.00908173194404882\\
7	0.0101243482897225\\
8	0.012315759612551\\
9	0.0152564124677886\\
10	0.0197990774496737\\
11	0.0201155503219162\\
12	0.0241947337700266\\
13	0.0261693760169298\\
14	0.0304848024511048\\
15	0.0318803575090226\\
16	0.0374762230488341\\
17	0.0405393155000847\\
18	0.0409921153109883\\
19	0.0428354352585637\\
20	0.04342029540049\\
21	0.0473667427511066\\
22	0.0474907082572089\\
23	0.0494288399262148\\
24	0.051809881869114\\
25	0.0524418624954034\\
26	0.0577185202311454\\
27	0.0585771363421718\\
28	0.0586952582398385\\
29	0.0605798753318845\\
30	0.0629132846749035\\
31	0.0645931020091887\\
32	0.0677014587530016\\
33	0.0688935060905428\\
34	0.0712329456637397\\
35	0.073036508995649\\
36	0.0742247060405337\\
37	0.0762871809052927\\
38	0.0811036361580574\\
39	0.0835110493744031\\
40	0.083573947518264\\
41	0.0842130650351083\\
42	0.0857699046436178\\
43	0.0859019083537373\\
44	0.0911346573618471\\
45	0.0912890081565526\\
46	0.0933777596343693\\
47	0.0935201760897036\\
48	0.0973950519142732\\
49	0.100843994873062\\
50	0.102956998955766\\
};
\addlegendentry{Best}

\end{axis}
\end{tikzpicture}%
\caption{The figure shows the accumulated prediction error as a function of the number of iterations. ``MW" refers to multiplicative weights and ``Best" to the best single action in hindsight. }\label{Fig:SisID}
\end{figure}
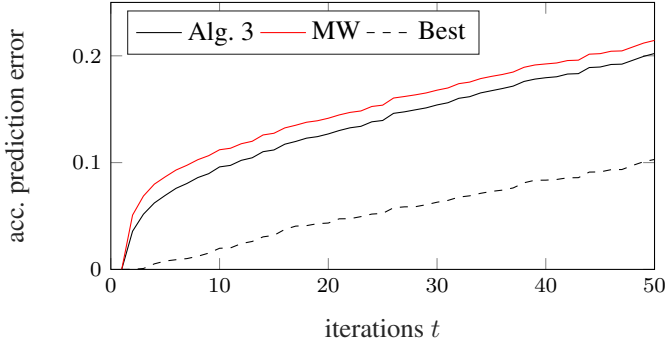

\section{Conclusion}\label{Sec:Conclusion}
The article analyzed adaptive decision-making problems, where an agent repeatedly chooses among $m$ options and aims at minimizing regret. We considered a setting where the losses are constrained, which provides the agent with additional information. We developed an algorithm that accounts for the constraints and is instance-adaptive, i.e., it exploits suboptimal choices of the adversary. The algorithm relies on the solution of an optimization problem at every iteration and largely outperforms multiplicative weights. We also discussed an approximation strategy to reduce the computational effort and highlighted an application to nonlinear and online system identification.

\bibliography{ifacconf}             
                                                   






\appendix
\section{Proof of Prop.~3.1}\label{App:proof}   
We again analyze $\phi_t=\sum_{i=1}^{m} w_{t}^i$, which we express as
\begin{align*}
\phi_t &= \phi_{t-1}^{\frac{\epsilon_{t}}{\epsilon_{t-1}}} \sum_{i=1}^{m} \left(\frac{w_{t-1}^i}{\phi_{t-1}}\right)^{\frac{\epsilon_t}{\epsilon_{t-1}}} e^{-\epsilon_t l_t^i}\\
=& \phi_{t-1}^{\frac{\epsilon_{t}}{\epsilon_{t-1}}} e^{-\epsilon_t \sum_{i=1}^{m} p_t^i l_t^i} \sum_{i=1}^{m} \left(\frac{w_{t-1}^i}{\phi_{t-1}} e^{-\epsilon_{t-1} (l_t^i- \sum_{j=1}^{m} p_t^j l_t^j)}\right)^{\frac{\epsilon_t}{\epsilon_{t-1}}}\!\!\!,
\end{align*}
for any $t\geq 0$. Due to the fact that $\epsilon_{t} \leq \epsilon_{t-1}$ we can invoke the following variant of Jensen's inequality
\begin{equation*}
\sum_{i=1}^{m} x_i^\alpha \leq m^{1-\alpha} \left(\sum_{i=1}^{m} x_i\right)^{\alpha},
\end{equation*}
which holds for any $x_1\geq 0,\dots, x_m \geq 0$ and $\alpha \in (0,1]$. We obtain therefore the following upper bound on $\phi_t$
\begin{equation*}
\phi_{t-1}^{\frac{\epsilon_{t}}{\epsilon_{t-1}}} e^{-\epsilon_t p_t^i l_t^i} m^{1-\frac{\epsilon_t}{\epsilon_{t-1}}}   \left( \underbrace{\sum_{i=1}^{m} \frac{w_{t-1}^i}{\phi_{t-1}} e^{-\epsilon_{t-1} (l_t^i- p_t^j l_t^j)}}_{\text{part~ii}}\right)^{\frac{\epsilon_t}{\epsilon_{t-1}}},
\end{equation*}
where we used Einstein's summation convention to simplify the notation. We apply Taylor's expansion to the exponential in the sum over $w_{t-1}^i$, which yields
\begin{multline*}
\text{part ii} =  1- \epsilon_{t-1} \sum_{i=1}^{m} \frac{w_{t-1}^i}{\phi_{t-1}} (l_t^i-p_t^j l_t^j) \\
+ \frac{\epsilon_{t-1}^2}{2} \sum_{i=1}^{m} \frac{w_{t-1}^i}{\phi_{t-1}} (l_t^i - p_t^j l_t^j)^2 + \frac{\epsilon_{t-1}^3}{6} \rho_t,
\end{multline*}
where the remainder term $\rho_t$ is bounded by $L^3 e^{L \epsilon_{t-1}}$. By choosing $p_t=w_{t-1}/\phi_{t-1} - \epsilon_{t-1} Q_{t-1} q_t^*/2$, see Alg.~\ref{Alg:EMW}, we thus obtain
\begin{equation*}
\text{part ii} \leq 1+ \frac{\epsilon_{t-1}^2}{2} \tilde{r}_t^* + \frac{\epsilon_{t-1}^3}{6} \rho_t \leq e^{\frac{\epsilon_{t-1}^2}{2} \tilde{r}_t^* + \frac{\epsilon_{t-1}^3}{6}\rho_t},
\end{equation*}
which gives rise to 
\begin{equation*}
\phi_t \leq \phi_{t-1}^{\frac{\epsilon_{t}}{\epsilon_{t-1}}} m^{1-\frac{\epsilon_t}{\epsilon_{t-1}}}  e^{-\epsilon_{t} p_t^i l_t^i + \frac{\epsilon_t \epsilon_{t-1}}{2} \tilde{r}_t^* + \frac{\epsilon_t \epsilon_{t-1}^2}{6} \rho_t},
\end{equation*}
for any $t\geq 0$. Thus, after unrolling the recursion we obtain
\begin{multline*}
\phi_{t} \leq \phi_{-1}^{\epsilon_t/\epsilon_{-1}} m^{ \epsilon_t \sum_{j=0}^{t} \frac{\epsilon_{j-1}- \epsilon_j}{\epsilon_{j-1} \epsilon_j} } \\ e^{-\epsilon_t \sum_{j=0}^{t} p_j^i l_j^i} e^{\epsilon_t \sum_{j=0}^{t} \frac{\epsilon_{j-1}}{2} \tilde{r}_j^* + \frac{\epsilon_{j-1}^2}{6} \rho_j}.
\end{multline*}
We further note that, by construction, $\phi_t \geq e^{-\epsilon_t \sum_{j=0}^{t} l_j^i}$ for any $i\in \{1,\dots,m\}$, which yields after rearranging terms
\begin{multline*}
\sum_{t=0}^{T-1} p_t\T l_t - \!\!\!\!\min_{i\in \{1,\dots,m\}} \!\sum_{t=0}^{T-1} l_t^i \leq \frac{\log(m)}{\epsilon_{-1}} + \log(m) \sum_{t=0}^{T-1} \frac{\epsilon_{t-1}-\epsilon_t}{\epsilon_{t-1}\epsilon_t} \\
+ \sum_{t=0}^{T-1} \frac{\epsilon_{t-1} \tilde{r}_t^*}{2} +\frac{\epsilon_{t-1}^2}{6} \rho_t.
\end{multline*}
We introduce the variable $s_{t}:=\sum_{j=0}^{t} \tilde{r}_j^*+c_2$, which means that $\epsilon_{t}$ can be expressed as $\sqrt{2\log(m)/s_t}$. As a result,
\begin{align}
\sum_{t=0}^{T-1} \frac{\epsilon_{t-1}-\epsilon_t}{\epsilon_{t-1} \epsilon_t} &= \frac{1}{\sqrt{2 \log(m)}} \sum_{t=0}^{T-1} \sqrt{s_t} - \sqrt{s_{t-1}} \label{eq:claim0}\\
&=\frac{\sqrt{s_{T-1}} - \sqrt{L^2/4}}{\sqrt{2 \log(m)}}=\sqrt{\frac{s_{T-1}}{2 \log(m)}} - \frac{\text{log}(m)}{\epsilon_{-1}}. \nonumber
\end{align}
We now prove the following claim.\\
\noindent\textbf{Claim 1:} 
\begin{equation}
\sum_{j=0}^{t} \frac{\tilde{r}_j^*}{\sqrt{s_{j-1}}} \leq 2 \sqrt{2 s_t}, \quad \forall t\geq 0. \label{eq:claim1}
\end{equation}
The claim is proved by induction. The claim can be easily verified for $t=0$. We therefore proceed with the induction step, that is, we assume that the claim holds for $t\geq 0$, and show that it is also satisfied for $t+1$. More precisely, the induction hypothesis implies
\begin{align*}
\sum_{j=0}^{t+1} \frac{\tilde{r}_j^*}{\sqrt{s_{j-1}}} =  \frac{\tilde{r}_{t+1}^*}{\sqrt{s_t}} + \sum_{j=0}^{t} \frac{\tilde{r}_j^*}{\sqrt{s_{j-1}}} \leq \frac{\tilde{r}_{t+1}^*}{\sqrt{s_t}} + 2 \sqrt{2 s_t}.
\end{align*}
The right-hand side can further be bounded as
\begin{align*}
\frac{\tilde{r}_{j+1}^*}{\sqrt{s_j}} + 2 \sqrt{2} \sqrt{s_{t+1} - \tilde{r}_{t+1}^*} \leq \frac{\tilde{r}_{j+1}^*}{\sqrt{s_j}} + 2 \sqrt{2 s_{t+1}} \left(1-\frac{\tilde{r}_{t+1}^*}{2s_{t+1}}\right),
\end{align*}
where we have used $s_t=s_{t+1}-\tilde{r}_{t+1}^*$ in the first step and $\sqrt{1-\xi} \leq 1-\xi/2$ for all $\xi \in [0,1]$ in the second step. We therefore obtain
\begin{align*}
\sum_{j=0}^{t+1} \frac{\tilde{r}_j^*}{\sqrt{s_{j-1}}} \leq 2 \sqrt{2 s_{t+1}} + \frac{\tilde{r}_{t+1}^*}{\sqrt{s_{t+1}}} \left( \frac{\sqrt{s_{t+1}}}{\sqrt{s_t}} - \sqrt{2}\right).
\end{align*}
Due to the fact that $\tilde{r}_j^* \leq L^2/4$ we obtain $s_{t+1}/s_t \leq 2$. This means the the claim holds true for $t+1$. By induction, we can therefore infer that the claim is true. \hfill \qed

We make a second claim, which we will prove subsequently.\\
\noindent\textbf{Claim 2:}
\begin{equation}
\sum_{j=0}^{t} \frac{1}{s_{j-1}} \leq \frac{1}{\bar{r}^*} \text{log}(4 s_t/L^2) + \frac{1}{c_2},\quad \forall t\geq 0. \label{eq:claim2}
\end{equation}
The claim is again proved by induction and we note that it holds (trivially) for $t=0$. We perform the step from $t$ to $t+1$ as follows
\begin{align*}
\sum_{j=0}^{t+1} &\frac{1}{s_{j-1}} \leq \frac{1}{s_{t+1}} + \frac{1}{\bar{r}^*} \log(4 s_t/L^2) + \frac{1}{c_2}\\
&=\!\frac{1}{s_{t+1}}\!+ \! \frac{1}{\bar{r}^*} \log(4s_{t+1}/L^2) \!+\! \frac{1}{\bar{r}^*} \log(1\!-\!\tilde{r}_{t+1}^*/s_{t+1})\!+\!\frac{1}{c_2}.
\end{align*}
Due to the fact that $\log(1-\xi) \leq -\xi$ for all $\xi\leq 1$, we obtain
\begin{equation*}
\sum_{j=0}^{t+1} \frac{1}{s_{j-1}} \leq \frac{1}{\bar{r}^*} \log(4 s_{t+1}/L^2) + \frac{1-\tilde{r}_{t+1}^*/\bar{r}^*}{s_{t+1}} + \frac{1}{c_2}.
\end{equation*}
The term $1-\tilde{r}_{t+1}^*/\bar{r}^*$ in the previous equation is negative, since $\tilde{r}_{t+1}^*\geq \bar{r}^*$, see Alg.~\ref{Alg:EMW}, which proves the claim. \hfill \qed

We can now combine all the ingredients (\eqref{eq:claim0}, \eqref{eq:claim1}, and \eqref{eq:claim2}) and obtain the following regret bound
\begin{multline*}
\sum_{t=0}^{T-1} p_t\T l_t - \!\!\!\!\min_{i\in \{1,\dots,m\}} \!\sum_{t=0}^{T-1} l_t^i \leq (\frac{\sqrt{2}}{2}+2) \sqrt{s_{T-1} \log(m)} \\+ \frac{\log(m) \exp(1) L^3}{3 c_2} + \frac{L^3 \exp(1)}{6 \bar{r}^*} \log(4 s_{T-1}/L^2).
\end{multline*}
We can further use the inequality $\sqrt{a+b} \leq \sqrt{a} + \sqrt{b}$ for all $a,b\geq 0$ and the fact that $c_2=2 \log(m) L^2 \geq 2 \log(2) L^2$ to rewrite the above bound as
\begin{multline*}
\sum_{t=0}^{T-1} p_t\T l_t - \!\!\!\!\min_{i\in \{1,\dots,m\}} \!\sum_{t=0}^{T-1} l_t^i \leq 3 \sqrt{\log(m) \sum_{j=0}^{T-1} \tilde{r}_{j}^*} + 3\sqrt{c_2 \log{m}}\\
+\frac{L}{2} + \frac{L^3}{2 \bar{r}^*} (\log(4s_{T-1}/L^2)).
\end{multline*}
The remaining steps are simple manipulations, which also rely on the fact that $4/s_{T-1}/L^2 \geq 8 \log(m)\geq 8 \log(2)$ and $L^2/4\geq \bar{r}^*$. \hfill \qed

\section{Proof of Prop.~4.1}\label{App:B}
We note that for small $\epsilon$, $\hat{q}_t=\mu_t$, and introduce the notation $l_t^{ij}:=\bar{l}_t^i-\underline{l}_t^j$. As a result, we obtain
\begin{multline}
\max_{l_t\in X_t} l_t\T Q_{t-1} (l_t-\hat{q}_t) = \\ \max_{l_t\in X_t} \sum_{i,j=1}^{m} \frac{w_{t-1}^i w_{t-1}^j}{2\phi_{t-1}^2} (l_t^i-l_t^j) (l_t^i- l_t^j - l_t^{ij}+ l_t^{ji}). \label{eq:tmptmp1}
\end{multline}
We further note that $-l_t^{ji}=\underline{l}_t^i-\bar{l}_t^j \leq l_t^i- l_t^j \leq \bar{l}_t^i-\underline{l}_t^j=l_t^{ij}$, and consider a single summand on the right-hand side of the above equation. We denote this summand as 
\begin{equation*}
f_{ij}(l_t^i-l_t^j):=\frac{w_{t-1}^i w_{t-1}^j}{2\phi_{t-1}^2} (l_t^i-l_t^j) (l_t^i- l_t^j - l_t^{ij}+ l_t^{ji}),
\end{equation*}
and conclude that $f_{ij}$ is convex in $l_t^i-l_t^j$. Thus, for any $\lambda\in [0,1]$ we have
\begin{equation*}
f_{ij}(l_t^i-l_t^j) \leq \lambda f_{ij}(-l_t^{ji}) + (1-\lambda) f_{ij}(l_t^{ij}),
\end{equation*}
where both $f_{ij}(-l_t^{ji})$ and $f_{ij}(l_t^{ij})$ evaluate to the same value $w_{t-1}^i w_{t-1}^j l_t^{ji} l_t^{ij}/2\phi_{t-1}^2$. By applying the same argument to all summands in \eqref{eq:tmptmp1}, we conclude
\begin{equation*}
\max_{l_t\in X_t} l_t\T Q_{t-1} (l_t-\hat{q}_t) \leq \sum_{i,j=1}^{m} \frac{w_{t-1}^i w_{t-1}^j}{2\phi_{t-1}^2} l_t^{ij} l_t^{ji},
\end{equation*}
which yields the desired result. \hfill \qed

\section{Positive semi-definiteness of $Q_{t-1}$}\label{App:C}
We note that for any $x\in \mathbb{R}^m$, $x\T Q_{t-1}x=\sum_{i=1}^m (x^i)^2 u^i - (\sum_{i=1}^m u^i x^i)^2$, where $u:=w_{t-1}/\phi_{t-1}$. According to Sedrakyan's inequality we have
\begin{equation*}
(\sum_{i=1}^m u^i x^i)^2 = \frac{(\sum_{i=1}^m u^i x^i)^2}{\sum_{i=1}^m u^i} \leq \sum_{i=1}^m u^i (x^i)^2,
\end{equation*}
which concludes that $Q_{t-1}$ is indeed positive semi-definite.
\end{document}